\documentclass[a4paper,11pt]{article}
\usepackage[top=4cm, bottom=4cm, left=3cm, right=3cm]{geometry}

\usepackage[utf8]{inputenc} 
\usepackage[T1]{fontenc}    
\usepackage{url}            
\usepackage{amsfonts}       

\usepackage{amssymb}
\usepackage{amsmath}
\usepackage{amsthm}

\usepackage{thmtools, thm-restate}
\usepackage{mathrsfs} 
\usepackage[usenames,dvipsnames]{pstricks}
\usepackage{euler}
\usepackage{euscript}
\usepackage{graphicx}
\usepackage{charter}
\usepackage{mdframed}
\usepackage{enumerate, color}
\usepackage{bibunits}
	
\usepackage{newfloat}

\usepackage{booktabs}
\usepackage{cite}

\usepackage{floatrow}
\usepackage{cleveref}

\usepackage{algorithm, algorithmic}

\usepackage[font=small,skip=0pt]{caption}

\newcommand{\X}{{\mathcal X}}
\newcommand{\Y}{{\mathcal Y}}

\newcommand{\fhat}{{\widehat f}}
\newcommand{\fstar}{{f^*}}

\newcommand{\ghat}{{\widehat g}}
\newcommand{\gstar}{{g^*}}

\newcommand{\E}{\mathcal{E}}

\renewcommand{\H}{\mathcal{H}}

\newcommand{\R}{\mathbb{R}}

\newcommand{\N}{\mathbb{N}}

\renewcommand{\L}{\loss}
\newcommand{\F}{\mathcal{F}}

\newcommand{\hh}{ {\mathcal{H}} }

\newcommand{\ppp}[1]{{P^{#1}_{++}}}

\renewcommand{\S}{\mathcal{S}}

\newcommand{\la}{\lambda}

\newcommand{\loss}{\bigtriangleup}

\renewcommand{\eqref}[1]{Eq.~(\ref{#1})}

\newcommand{\rhox}{{\rho_\X}}

\newcommand{\frho}{\fstar}

\newcommand{\eqals}[1]{\begin{align*}#1\end{align*}}

\newcommand{\eqal}[1]{\begin{align}#1\end{align}}

\renewcommand{\eqals}[1]{\eqal{#1}}

\providecommand{\scal}[2]{\left\langle{#1},{#2}\right\rangle}

\newcommand{\cutlocus}{\ensuremath{\text{\rm Cut}}}

\newcommand{\supp}{\ensuremath{\text{\rm supp}}}
\newcommand{\tr}{\ensuremath{\text{\rm Tr}}}

\newcommand{\argmin}{\operatornamewithlimits{argmin}}

\newcommand{\HS}{{\rm HS}}

\newcommand{\ccinf}{{C^\infty_c}}
\newcommand{\cinf}{{C^\infty}}
\newcommand{\wm}{{\widetilde{\M}}}

\newcommand{\geod}{{d}}
\newcommand{\M}{{\cal M}}

\declaretheorem[name=Theorem,refname=Thm.]{theorem}
\declaretheorem[name=Lemma,sibling=theorem]{lemma}

\declaretheorem[name=Remark]{remark}

\declaretheorem[name=Definition,refname=Def.]{definition}

\declaretheorem[name=Assumption,refname=Asm.]{assumption}
\crefname{assumption}{assumption}{assumptions}

\crefname{equation}{Eq.}{Eqs.}
\crefname{figure}{Fig.}{Figs.}
\crefname{table}{Tab.}{Tabs.}
\crefname{section}{Sec.}{Secs.}

\title{\sffamily\huge\bf Manifold Structured Prediction}

\author{\small Alessandro Rudi$^{1*}$ \\ {\scriptsize alessandro.rudi@inria.fr} \and \small Carlo Ciliberto$^{2*}$ \\ {\scriptsize c.ciliberto@ucl.ac.uk} \and \small Gian Maria Marconi$^{3,5}$ \\ {\scriptsize gianmaria.marconi@iit.it} \and \small Lorenzo Rosasco$^{4,5}$ \\ {\scriptsize lrosasco@mit.edu }\\ \and \small ${}^*$Equal contribution}

\begin{document}

\maketitle

\begin{abstract}
\noindent Structured prediction provides a general framework to deal with supervised problems where the outputs have semantically rich structure. While classical approaches consider finite, albeit potentially huge, output spaces, in this paper we discuss how structured prediction can be extended to a continuous scenario. Specifically, we study a structured prediction approach to manifold valued regression. We characterize a class of problems for which the considered approach is statistically consistent and study how geometric optimization can be used to compute the corresponding estimator. Promising experimental results on both simulated and real data complete our study.
\end{abstract}

\section{Introduction}

\footnotetext[1]{INRIA - Département d’informatique, Ecole Normale Supérieure - PSL Research University, Paris, France.}\footnotetext[2]{University College London, WC1E 6BT London, United Kingdom}\footnotetext[3]{iCub Facility, Istituto Italiano di Tecnologia, Via Morego, 30, Genoa 16163, Italy}\footnotetext[4]{Laboratory for Computational and Statistical Learning - Istituto Italiano di Tecnologia, Genova, Italy \& Massachusetts Institute of Technology, Cambridge, MA 02139, USA.}\footnotetext[5]{Universit\'a degli Studi di Genova, Genova, Italy}
Regression and classification are probably the most classical machine learning problems and correspond to estimating a function with scalar and binary values, respectively. In practice, it is often interesting to estimate functions with more structured outputs. When the output space can be assumed to be a vector space, many ideas from regression can be extended, think for example to multivariate \cite{hardle2007applied} or functional regression \cite{morris2015functional}. However, a lack of a natural vector structure is a feature of many practically interesting problems, such as ranking \cite{duchi2010consistency}, quantile estimation \cite{le2005nonparametric} or graph prediction \cite{paassen2017time}. In this latter case, the outputs are typically provided only with some distance or similarity function that can be used to design appropriate loss function. Knowledge of the loss is sufficient to analyze an abstract empirical risk minimization approach within the framework of statistical learning theory, but deriving  approaches that are at the same time statistically sound and computationally feasible is a key challenge. While ad-hoc solutions are available for many specific problems \cite{daume2006practical,nowozin2011structured,kadous2005classification,bicer2011relational}, structured prediction \cite{bakir2007predicting} provides a unifying framework where a variety of problems can be tackled as special cases. 

Classically, structured prediction considers problems with finite, albeit potentially huge, output spaces. In this paper, we study how these ideas can be applied to non-discrete output spaces. In particular, we consider the case where the output space is a Riemannian manifold, that is the problem of manifold structured prediction (also called manifold valued regression~\cite{steinke2009non}). While also in this case ad-hoc methods are available~\cite{steinke2010nonparametric}, in this paper we adopt and study a structured prediction approach starting from a framework  proposed in \cite{ciliberto2016}. Within this framework, it is possible to derive a statistically sound, and yet computationally feasible, structured prediction approach, as long as the loss function satisfies a suitable structural assumption. Moreover we can guarantee that the computed prediction is always an element of the manifold.

Our main technical contribution is a characterization of loss functions for manifold structured prediction satisfying such a structural assumption. In particular, we consider the case where the Riemannian metric is chosen as a loss function. As a byproduct of these results, we derive a manifold structured learning algorithm that is universally consistent and corresponding finite sample bounds. From a computational point of view, the proposed algorithm requires solving a linear system (at training time) and a minimization problem over the output manifold (at test time). To tackle this latter problem, we investigate the application of geometric optimization methods, and in particular Riemannian gradient descent \cite{absil2009optimization}. We consider both numerical simulations and benchmark datasets reporting promising performances. The rest of the paper is organized as follows. In Section~\ref{sec:problem-setting}, we define the problem and explain the proposed algorithm. In Section~\ref{sec:theory} we state and prove the theoretical results of this work. In Section~\ref{sec:algorithms} we explain how to compute the proposed algorithm and we show the performance of our method on synthetic and real data.

\section{Structured Prediction for Manifold Valued Regression}\label{sec:problem-setting}
The goal of supervised learning is to find a functional relation between an input space $\X$ and an output space $\Y$ given a finite set of observations. Traditionally, the output space is either a linear space (e.g. $\Y = \R^M$) or a discrete set (e.g. $\Y = \{0, 1\}$ in binary classification).  In this paper, we consider the problem of manifold structured prediction~\cite{steinke2010nonparametric}, in which  output data lies on a manifold $\M \subset \R^{d}$. In this context, statistical learning corresponds to solving
\begin{equation}\label{eq:statproblem}
\argmin_{f \in \X \to \Y} \mathcal{E}(f) \qquad \textrm{with} \qquad \E(f) = \int_{\X \times \Y} \L( f(x), y ) \: \rho(x,y)
\end{equation}
where $\Y$ is a subset of the manifold $\M$ and $\rho$ is an unknown distribution on $\X\times\Y$. Here, $\L \colon \Y \times \Y \to \R$ is a loss function that measures prediction errors for points estimated on the manifold. The minimization is meant over the set of all measurable functions from $\X$ to $\Y$. The distribution is fixed but unknown  and  a  learning algorithm seeks an estimator $\fhat:\X\to\Y$ that approximately solves  \Cref{eq:statproblem}, given a set of training points $(x_i, y_i)_{i=1}^{n}$ sampled independently from $\rho$. 

A concrete example of loss function that we will consider in this paper is $\loss = \geod^2$ the squared geodesic distance $\geod:\Y\times\Y\to\R$ \cite{lee2003smooth}. The geodesic distance is the natural metric on a Riemannian manifold (it corresponds to the Euclidean distance when $\M = \R^d$) and is a natural loss function in the context of manifold regression \cite{steinke2009non,steinke2010nonparametric,fletcher2013geodesic, hinkle2012polynomial, hauberg2012geometric}.

\subsection{Manifold Valued Regression via Structured Prediction}\label{sec:estimator}
In this paper we consider a structured prediction approach to manifold valued regression following ideas  in \cite{ciliberto2016}. Given a training set $(x_i,y_i)_{i=1}^n$, an estimator for  problem  \Cref{eq:statproblem} is defined by 
\begin{equation}\label{eq:structestimator}
\fhat(x) =  \argmin_{y \in \Y} \sum\limits_{i=1}^{n} \alpha_i(x) \L(y, y_i)
\end{equation}
for any  $x\in\X$. The coefficients $\alpha(x) = (\alpha_1(x),\dots,\alpha_n(x))^\top\in\R^n$ are obtained solving a linear system for a problem akin to kernel ridge regression (see \Cref{sec:estimator-derivation}): given a positive definite kernel $k:\X\times\X\to\R$ \cite{aronszajn1950theory} over $\X$, we have 
\eqal{\label{eq:alpha-scores}
	\alpha(x) = (\alpha_1(x),\dots,\alpha_n(x))^\top = (K + n\la I)^{-1} K_x
}
where $K \in \R^{n \times n}$ is the empirical kernel matrix with $K_{i,j} = k(x_i, x_j) $, and $K_x \in \R^n$ the vector whose $i$-th entry corresponds to $(K_x)_i = k(x, x_i)$. Here, $\lambda \in \R_+$ is a regularization parameter and $I \in \R^{n \times n}$ denotes identity matrix. 

Computing the estimator in \Cref{eq:structestimator} can be divided into two steps. During a {\em training step} the score function $\alpha:\X\to\R^n$ is learned, while during the {\em prediction step}, the output $\fhat(x)\in\Y$ is estimated on a new test point $x\in\X$. This last step requires minimizing the linear combination of distances $\loss(y,y_i)$ between a candidate $y\in\Y$ and the training outputs $(y_i)_{i=1}^n$, weighted by the corresponding scores $\alpha_i(x)$. Next,  we recall the derivation of the above estimator following \cite{ciliberto2016}. 

\subsection{Derivation of the Proposed Estimator}\label{sec:estimator-derivation}

The derivation of the estimator $\fhat$ in \Cref{eq:structestimator} is based on the following key structural assumption on the loss.

\begin{definition}[Structure Encoding Loss Function (SELF)]\label{def:SELF}
	Let $\Y$ be a compact set. A function $\L: \Y \times \Y \to \R$ is a {Structure Encoding Loss Function} if there exist a separable Hilbert space $\hh$, a continuous feature map $\psi: \Y \to \hh$ and a continuous linear operator $V: \hh \to \hh$ such that for all $y, y' \in \Y$
	\eqal{\label{eq:self}
		\L(y,y') = \langle \psi(y), V\psi(y')\rangle_{\H}.
	}
\end{definition}
\noindent Intuitively, the SELF definition requires a loss function to be ``bi-linearizable'' over the space $\hh$. This is similar, but more general, than requiring the loss to be a kernel since it allows also to consider distances (which are not positive definite) or even non-symmetric loss functions. As observed in  \cite{ciliberto2016},  a wide range of loss functions often used in machine learning are SELF. In \Cref{sec:theory} we  study how the above assumption applies to manifold structured loss functions, including the squared geodesic distance. 

We first recall  how the estimator \Cref{eq:structestimator} can be obtained assuming $\loss$ to be SELF. We begin by rewriting the expected risk in \Cref{eq:statproblem} as 
\eqal{
	\mathcal{E}(f) = \int_{\X} \scal{\psi(f(x))}{V ~\int_{\Y} \psi(y) ~ d\rho(y|x)}_{\H} ~ d\rho(x)
}
where we have conditioned $\rho(y,x) = \rho(y|x)\rhox(x)$ and used the linearity of the integral and the inner product. Therefore, any function $\fstar: \X \to \Y$ minimizing the above functional must satisfy the following condition
\begin{align}\label{eq:fstar-from-gstar}
\fstar(x) = \argmin_{y \in \Y} \scal{\psi(y)}{V \gstar(x)}_\hh \qquad \mbox{where} \qquad g^*(x) = \int_{\Y} \psi(y) \; d\rho(y|x)
\end{align}
where we have introduced the function $\gstar:\X\to\hh$ that maps each point $x\in\X$ to the conditional expectation of $\psi(y)$ given $x$. 
However we cannot compute explicitly $\gstar$, but noting that it minimizes the expected least squares error
\eqal{
	\int \|\psi(y)-g(x)\|_\hh^2d\rho(x,y)
}
suggests that a least squares estimator can be considered. We first  illustrate this idea for $\X = \R^d$ and $\hh = \R^k$.
In this case we can consider a ridge regression estimator
\begin{align}
\ghat(x) = \widehat{W}^\top x \qquad \mbox{with} \qquad \widehat{W} = \argmin_{W \in \R^{d \times k}} \frac{1}{n} \| X W - \psi(Y) \|_F^2 + \lambda \|W\|_F^2
\end{align}
where $X = (x_1,\dots,x_n)^\top \in\R^{n \times d}$ and $\psi(Y)=(\psi(y_1),\dots,\psi(y_n))^\top\in\R^{n \times k}$ are the matrices whose $i$-th row correspond respectively to the training sample $x_i\in\X$ and the (mapped) training output $\psi(y_i)\in\hh$. We have denoted $\|\cdot\|_F^2$ the squared Frobenius norm of a matrix, namely the sum of all its squared entries.
The ridge regression solution can be obtained in closed form as $\widehat W =(X^\top X + n\la I)^{-1} X^\top \psi(Y)$. For any $x\in\X$ we have 
\begin{align}\label{eq:ghat}
\ghat(x) = \psi(Y)^\top X (X^\top X + n\la I)^{-1} x = \psi(Y)^\top \alpha(x) = \sum_{i=1}^n \alpha_i(x) \psi(y_i)
\end{align}
where we have introduced the coefficients $\alpha(x) = X (X^\top X + n\la I)^{-1} x \in \R^n$. By substituting $\ghat$ to $\gstar$ in \Cref{eq:fstar-from-gstar} we have
\begin{align}
\fhat(x)  &= \argmin_{y \in \M} \scal{\psi(y)}{V \left(\sum_{i=1}^n \alpha_i(x)\psi(y_i)\right)} = \argmin_{y \in \M} \sum\limits_{i=1}^{n} \alpha_i(x) \L(y, y_i)
\end{align}
where we have used the linearity of the sum and the inner product to move the coefficients $\alpha_i$ outside of the inner product. Since the loss is SELF, we then obtain $\scal{\psi(y)}{V\psi(y_i)} = \loss(y,y_i)$ for any $y_i$ in the training set. This recovers the estimator $\fhat$ introduced in \Cref{eq:structestimator}, as desired.

We end noting how the above idea can be extended. First, we can consider  $\X$ to be a  set and  $k:\X\times\X\to\R$ a positive definite kernel. Then  $\ghat$ can be computed by  kernel ridge regression  (see e.g. \cite{shawe2004kernel}) to obtain the scores $\alpha(x) = (K + n \la I)^{-1}  K_{x}$, see  \Cref{eq:alpha-scores}. Second,  the above discussion applies if $\hh$ is infinite dimensional. Indeed, thanks to the SELF assumption, $\fhat$ does not depend on explicit knowledge of the space $\hh$ but only on the loss function.

We next discuss the main results of the paper, showing that a large class of loss functions for manifold structured prediction are SELF. 
This will allow us to prove consistency and learning rates for the manifold structured  estimator considered in this work.

\section{Characterization of SELF Function on Manifolds}\label{sec:theory}

In this section we provide sufficient conditions for a wide class of functions on manifolds to satisfy the definition of SELF. A key example will be the case of the squared geodesic distance. To this end we will make the following assumptions on the manifold $\M$ and the output space $\Y\subseteq\M$ where the learning problem takes place.

\begin{assumption}\label{asm:without-boundary}
	$\M$ is a complete $d$-dimensional smooth connected Riemannian manifold, without boundary, with Ricci curvature bounded below and positive injectivity radius. 
\end{assumption}

\noindent The assumption above imposes basic regularity conditions on the output manifold. In particular we require the manifold to be locally diffeomorphic to $\R^d$ and that the tangent space of $\M$ at any $p\in\M$ varies smoothly with respect to $p$. This assumption avoids pathological manifolds and is satisfied for instance by any smooth compact manifold (e.g. the sphere, torus,  etc.) \cite{lee2003smooth}. Other notable examples are the statistical manifold (without boundary) \cite{amari2007methods} any open bounded sub-manifold of the cone of positive definite matrices, which is often studied in geometric optimization settings \cite{absil2009optimization}. This assumption will be instrumental to guarantee the existence of a space of functions $\hh$ on $\M$ rich enough to contain the squared geodesic distance. 
%
%
\begin{assumption}\label{asm:Y-geodesically-convex}
	$\Y$ is a compact geodesically convex subset of the manifold $\M$.
\end{assumption}
\noindent A subset $\Y$ of a manifold is geodesically convex if for any two points in $\Y$ there exists one and only one minimizing geodesic curve connecting them. The effect of \Cref{asm:Y-geodesically-convex}  is twofold. On one hand it guarantees a generalized notion of convexity for the space $\Y$ on which we will solve the optimization problem in \Cref{eq:structestimator}. On the other hand it avoids the geodesic distance to have singularities on $\Y$ (which is key to our main result below). For a detailed introduction to most definitions and results reviewed in this section we refer the interested reader to standard references for differential and Riemannian geometry (see e.g. \cite{lee2003smooth}). We are ready to prove the main result of this work.

\begin{theorem}[Smooth Functions are SELF]\label{thm:smooth-self}
	Let $\M$ satisfy \Cref{asm:without-boundary} and $\Y\subseteq\M$ satisfy \Cref{asm:Y-geodesically-convex}. Then, any smooth function $h:\Y\times\Y\to\R$ is SELF on $\Y$.
\end{theorem}

\begin{proof}[Sketch of the proof (\Cref{thm:smooth-self})]
The complete proof of \Cref{thm:smooth-self} is reported in \Cref{sect:proof}. The proof hinges around the following key steps:\\

\noindent{\bf Step~1 If there exists an RKHS $\hh$ on $\M$, then any $h \in \hh \otimes \hh$ is SELF.} Let $\hh$ be a reproducing kernel Hilbert space (RKHS) \cite{aronszajn1950theory} of functions on $\M$ with associated bounded kernel $k:\M\times\M\to\R$. Let $\hh\otimes\hh$ denote the RKHS of functions $h:\M\times\M\to\R$ with associated kernel $\bar k$ such that $\bar k((y,z),(y',z')) = k(y,y')k(z,z')$ for any $y,y',z,z'\in\M$. Let, $h:\M\times\M\to\R$ be such that $h\in\hh\otimes\hh$. Recall that $\hh\otimes\hh$ is isometric to the space of Hilbert-Schmidt operators from $\hh$ to itself. Let $V_h:\hh\to\hh$ be the operator corresponding to $h$ via such isometry. We show that the SELF definition is satisfied with $V = V_h$ and $\psi(y) = k(y,\cdot)\in\hh$ for any $y\in\M$. In particular, we have $\|V\| \leq \|V\|_{\HS} = \|h\|_{\hh\otimes\hh}$, with $\|V\|_\HS$ denoting the Hilbert-Schmidt norm of $V$.\\

\noindent{\bf Step~2: Under \Cref{asm:Y-geodesically-convex}, $\ccinf(\M)\otimes\ccinf(\M)$ ``contains'' $\cinf(\Y \times \Y)$.} If $\Y$ is compact and geodesically convex, then it is diffeomorphic to a compact set of $\R^d$. By using this fact, we prove that any function in $\cinf(\Y \times \Y)$, the space of smooth functions on $\Y\times\Y$, admits an extension in $\ccinf(\M\times\M)$ the space of smooth functions on $\M \times \M$ vanishing at infinity (this is well defined since $\M$ is diffeomorphic to $\R^d$ thanks to \Cref{asm:M-allows-W-to-be-RKHS}), and that $\ccinf(\M\times\M) = \ccinf(\M)\otimes\ccinf(\M)$.\\

\noindent{\bf Step~3: Under \Cref{asm:without-boundary}, there exists an RKHS on $\M$ containing $\ccinf(\M)$.} Under \Cref{asm:without-boundary}, the Sobolev space $\hh = H_s^2(\M)$ of square integrable functions with smoothness $s$ is an RKHS for any $s>d/2$ (see \cite{hebey2000nonlinear} for a definition of Sobolev spaces on Riemannian manifolds).\\

\noindent The proof proceeds as follows: from Step~1, we see that to guarantee $h$ to be SELF it is sufficient to prove the existence of an RKHS $\hh$ such that $h\in\hh\otimes\hh$. The rest of the proof is therefore devoted to showing that for smooth functions this is satisfied for $\hh = H_s^2(\M)$. Since $h$ is smooth, by Step~2 we have that under \Cref{asm:Y-geodesically-convex}, there exists a $\bar h\in\ccinf(\M)\otimes\ccinf(\M)$ whose restriction $\bar h|_{\Y\times\Y}$ to $\Y\times\Y$ corresponds to $h$. Now, denote by $H_s^2(\M)$ the Sobolev space of squared integrable functions on $\M$ with smoothness index $s > 0$. By construction, (see \cite{hebey2000nonlinear}) for any $s>0$, we have $\ccinf(\M)|_\Y \subseteq H_s^2(\M)|_\Y$, namely for any function. In particular, $\bar h\in \ccinf(\M)\otimes\ccinf(\M) \subseteq H_s^2(\M)\otimes H_s^2(\M)$. Finally, Step~3 guarantees that under \Cref{asm:without-boundary}, $\hh = H_s^2(\M)$ with $s>d/2$ is an RKHS, showing that $h\in\hh\otimes\hh$ as desired.
\end{proof}

\noindent Interestingly, \Cref{thm:smooth-self} shows that the SELF estimator proposed in \Cref{eq:structestimator} can tackle {\em any} manifold valued learning problem in the form of \Cref{eq:statproblem} with smooth loss function. In the following we study the specific case of the squared geodesic distance.

\begin{theorem}[$\geod^2$ is SELF]\label{thm:geodesic-self}
	Let $\M$ satisfy \Cref{asm:without-boundary} and $\Y\subseteq\M$ satisfy \Cref{asm:Y-geodesically-convex}. Then, the squared geodesic distance $\loss = \geod^2:\M\times\M\to\R$ is smooth on $\Y$. Therefore $\loss$ is SELF on $\Y$.
\end{theorem}

\noindent The proof of the result above is reported in the supplementary material. The main technical aspect is to show that regularity provided by \Cref{asm:Y-geodesically-convex} guarantees the squared geodesic distance to be smooth. The fact that $\loss$ is SELF is then an immediate corollary of \Cref{thm:smooth-self}.

\subsection{Statistical Properties of Manifold Structured Prediction}

In this section, we characterize the generalization properties of the manifold structured estimator \Cref{eq:structestimator} in light of \Cref{thm:smooth-self} and \Cref{thm:geodesic-self}.

\begin{theorem}[Universal Consistency]
	Let $\M$ satisfy \Cref{asm:without-boundary} and $\Y\subseteq\M$ satisfy \Cref{asm:Y-geodesically-convex}. Let $\X$ be a compact set and $k: \X \times \X \to \R$ be a bounded continuous universal kernel\footnote{This is standard assumption for universal consistency (see \cite{steinwart2008support}). An example of continuous universal kernel on $\X=\R^d$ is the Gaussian $k(x,x') = \exp(-\|x-x'\|^2/\sigma)$, for $\sigma > 0$.} For any $n \in \N$ and any distribution $\rho$ on $\X \times \Y$ let $\fhat_n:\X \to \Y$ be the manifold structured estimator in \cref{eq:structestimator} for a learning problem with smooth loss function $\loss$, with $(x_i,y_i)_{i=1}^n$ training points independently sampled from $\rho$ and $\la_n = n^{-1/4}$. Then
	\eqals{
		\lim_{n\to \infty} {\cal E}(\fhat_n) = {\cal E}(\frho) \quad \textrm{with probability} ~ 1.
	}
\end{theorem}

\noindent The result above follows from Thm. $4$ in \cite{ciliberto2016} combined with our result in \Cref{thm:smooth-self}. It guarantees that the algorithm considered in this work finds a consistent estimator for the manifold structured problem, when the loss function is smooth (thus also in the case of the squared geodesic distance). As it is standard in statistical learning theory, we can impose regularity conditions on the learning problem, in order to derive also generalization bounds for $\fhat$. In particular, if we denote by $\F$ the RKHS associated to the kernel $k$, we will require $\gstar$ to belong to the same space $\hh\otimes\F$ where the estimator $\ghat$ introduced in \Cref{eq:ghat} is learned. In the simplified case discussed in \Cref{sec:estimator-derivation}, with linear kernel on $\X=\R^d$ and $\hh=\R^k$ finite dimensional, we have $\F=\R^d$ and this assumption corresponds to require the existence of a matrix $W_*^\top\in\R^{k \times d} = \hh\otimes\F$, such that $\gstar(x) = W_*^\top x$ for any $x\in\X$. In the general case, the space $\hh\otimes\F$ extends to the notion of {\em reproducing kernel Hilbert space for vector-valued functions} (see e.g. \cite{micchelli2005learning,alvarez2012kernels}) but the same intuition applies.

\begin{restatable}[Generalization Bounds]{theorem}{TBounds
	}\label{thm:generalization-bound}
	Let $\M$ satisfy \Cref{asm:without-boundary} and $\Y\subseteq\M$ satisfy \Cref{asm:Y-geodesically-convex}. Let $\hh = H_s^2(\M)$ with $s>d/2$ and $k: \X \times \X \to \R$ be a bounded continuous reproducing kernel with associated RKHS $\F$. For any $n \in \N$, let $\fhat_n$ denote the manifold structured estimator in \cref{eq:structestimator} for a learning problem with smooth loss $\loss:\Y\times\Y\to\R$ and $\la_n = n^{-1/2}$. If the conditional mean $\gstar$ belongs to $\hh\otimes\F$, then 
	\eqal{
		\E(\fhat_n)-\E(\fstar) \leq \mathsf{c}_\loss \mathsf{q}~\tau^2 ~ n^{-\frac{1}{4}}
	}
	holds with probability not less than $1-8e^{-\tau}$ for any $\tau>0$, with $\mathsf{c}_\loss = \|\loss\|_{\hh\otimes\hh}$ and $\mathsf{q}$ a constant not depending on $n, \tau$ or the loss $\loss$.
\end{restatable}

\noindent The generalization bound of \Cref{thm:generalization-bound} is obtained by adapting Thm. $5$ of \cite{ciliberto2016} to our results in \Cref{thm:smooth-self} as detailed in the supplementary material. To our knowledge these are the first results characterizing in such generality the generalization properties of an estimator for manifold structured learning with generic smooth loss function. We conclude with a remark on a key quantity in the bound of \Cref{thm:generalization-bound}.

\begin{remark}[The constant $\mathsf{c}_\loss$]\label{rem:cdelta}
	We comment on the role played in the learning rate by $\mathsf c_\loss$,  the norm of the loss function $\loss$ seen an element of the Hilbert space $\hh\otimes\hh$. Indeed, from the discussion of \Cref{thm:smooth-self} we have seen that any smooth function on $\Y$ is SELF and belongs to the set $\hh\otimes\hh$ with $\hh = H_{s}^2(\M)$, the Sobolev space of squared integrable functions for $s>d/2$. Following this interpretation, we see that the bound in \Cref{thm:generalization-bound} can improve significantly (in terms of the constants) depending on the regularity of the loss function: smoother loss functions will result in ``simpler'' learning problems and vice-versa. In particular, when $\loss$ corresponds to the squared geodesic distance, the more ``regular'' is the manifold $\M$, the learning problem will be. A refined quantitative characterization of $\mathsf c_\loss$ in terms of the Ricci curvature and the injective radius of the manifold is left to future work.
\end{remark}

\section{Manifold Structured Prediction Algorithm and Experiments}\label{sec:algorithms}
In this section we recall geometric optimization algorithms that can be adopted to perform the estimation of $\fhat$ on a novel test point $x$. We then evaluate the performance of the proposed method in practice, reporting  numerical results on simulated and real data. 
\subsection{Optimization on Manifolds}
We begin discussing the computational aspects related to evaluating the manifold structured estimator. In particular, we discuss how to address the optimization problem in \Cref{eq:structestimator} in specific settings. Given a test point $x\in\X$, this process consists in solving a minimization over $\Y$, namely
\begin{align}
\underset{y \in \mathcal{Y}}{\min} F(y)
\end{align}
where $F(y)$ corresponds to the linear combination of $\loss(y,y_i)$ weighted by the scores $\alpha_i(x)$ computed according to \Cref{eq:alpha-scores}. If $\Y$ is a linear manifold or a subset of $\M = \R^{d}$, this problem can be solved by means of gradient-based minimization algorithms, such as Gradient Descent (GD):
\begin{align}
y_{t+1} = y_t - \eta_t \nabla F(y_t)
\end{align}
for a step size $\eta_t \in \R$. This algorithm can be extended to \textit{Riemannian gradient descent} (RGD) \cite{zhang2016first} on manifolds, as 
\begin{align}\label{eq:RGD}
y_{t+1} = Exp_{y_t}( \eta_t \nabla_{\M} F(y_t))
\end{align}

\begin{table}[t]
	\begin{tabular}{ccc}
		\toprule
		  {} &   {Positive definite matrix manifold ($\ppp m$)} & Sphere ($\S_{d-1}$)  \\
		\midrule
		  {$F(y)$}
		&  {$\sum\limits_{i=1}^{n} \alpha_i \| \log(Y^{- \frac{1}{2}} Z_i Y^{- \frac{1}{2}}) \|_{F}^{2}$} 
		&$\sum\limits_{i=1}^{n} \alpha_i \arccos\left(\scal{z_i}{y}\right)^2 $\\
		  {$\nabla_{\M} F(y) $}
		&  {$2 \sum\limits_{i=1}^{n} \alpha_i Y^{ \frac{1}{2} } \log(Y^{ \frac{1}{2} } Z_i^{-1} Y^{ \frac{1}{2} }) Y^{\frac{1}{2}}$}
		&$4\sum_{i=1}^n \alpha_i (yy^{T} - I) \frac{\arccos(\langle z_i, y\rangle)}{\sqrt{1 - \langle z_i, y \rangle}} z_{i}$   \\
		  {$R_{y}(v)$}
		&  {$ Y^{\frac{1}{2}} \exp(Y^{-\frac{1}{2}} v Y^{-\frac{1}{2}} ) Y^{\frac{1}{2}}$}
		&$\frac{v}{\|v\|}$  \\
		\bottomrule
	\end{tabular}
	\caption{Structured loss, gradient of the structured loss and retraction for $\ppp m$ and $S_{d-1}$. $Z_i \in \ppp m$ and $z_i \in \S_{d-1}$ are the training set points. $I \in \R^{d \times d}$ is the identity matrix.\label{tab:manifoldinfo}}
\end{table}

Where $\nabla_{\M} F$ is the gradient defined with respect to the Riemannian metric (see \cite{absil2009optimization}) and $Exp_{y}:T_y \M \to\M$ denotes the exponential map on $y\in\Y$, mapping a vector from the tangent space $T_y \M$ to the associated point on the manifold according to the Riemannian metric \cite{lee2003smooth}. For completeness, the algorithm is recalled in \Cref{sec:algs-app}.
For this family of gradient-based algorithms it is possible to substitute the exponential map with a retraction $R_{y} \colon T_y \M \to \M$, which is a first order approximation to the exponential map. Retractions are often faster to compute and still offer convergence guarantees \cite{absil2009optimization}. In the following experiments we will use both retractions and exponential maps. We mention that the step size $\eta_t$ can be found with a line search over the validation set, for more see \cite{absil2009optimization}.

\Cref{tab:manifoldinfo} reports gradients and retraction maps for the geodesic distance of two problems of interest considered in this work: positive definite manifold and the sphere. See \Cref{sec:exp-positive-definite,sec:exp-fingerprints} for more details on the related manifolds. 

We point out that using optimization algorithms that comply with the geometry of the manifold, such as RGD, guarantees that the computed value is an element of the manifold. This is in contrast with algorithms that compute a solution in a linear space that contains $\M$ and then need to project the computed solution onto $\M$. We  next discuss  empirical evaluations of the proposed manifold structured estimator on both synthetic and real datasets.
%

\subsection{Synthetic Experiments: Learning Positive Definite Matrices}\label{sec:exp-positive-definite}
We consider the problem of learning a function $f:\R^d\to\Y=\ppp m$, where $\ppp m$ denotes the {\em cone of positive definite (PD)} $m \times m$ matrices. Note that $\ppp m$ is a manifold with squared geodesic distance $\loss_{\textrm{PD}}$ between any two PD matrices $Z,Y\in \ppp m$ defined as
\eqals{
	\loss_{\textrm{PD}}(Z,Y) = \|\log(Y^{-\frac{1}{2}} Z ~Y^{-\frac{1}{2}})\|_F^2  
}
where, for any $M\in\ppp m$, we have that $M^{\frac{1}{2}}$ and $\log(M)$ correspond to the matrices with same eigenvectors of $M$ but with respectively the square root and logarithm of the eigenvalues of $M$. In Table \ref{tab:manifoldinfo} we show the computation of the structured loss, the gradient of the structured loss and the exponential map of the PD matrix manifold. We refer the reader to \cite{moakher2006symmetric, bhatia2009positive} for a more detailed introduction on the manifold of positive definite matrices.

\begin{table}[t]
\footnotesize
	\newcommand{\sss}[2]{$#1$\scriptsize{$\pm#2$}}
	\centering
	\begin{tabular}{lcccc}
		\toprule
		& \multicolumn{2}{c}{\bf Squared loss} & \multicolumn{2}{c}{\bf $\loss_{\rm PD}$ loss}\\
		Dim & \small{\bf KRLS}   &\small{\bf SP}  &\small{\bf KRLS} &\small{\bf SP} \\
		\midrule
		5   &\sss{0.72}{0.08}   &\sss{0.89}{0.08}  &\sss{111}{64}  &\sss{0.94}{0.06} \\
		10  &\sss{0.81}{0.03} &\sss{0.92}{0.05} &\sss{44}{8.3}  &\sss{1.24}{0.06}\\
		15  &\sss{0.83}{0.03}   &\sss{0.91}{0.06}  &\sss{56}{10} &\sss{1.25}{0.05}\\
		20  &\sss{0.85}{0.02} &\sss{0.91}{0.03} &\sss{59}{12} &\sss{1.33}{0.03}\\
		25  &\sss{0.87}{0.01}   &\sss{0.91}{0.02}  &\sss{72}{9}  &\sss{1.44}{0.03}\\
		30  &\sss{0.88}{0.01}   &\sss{0.91}{0.02}  &\sss{67}{7.2}  &\sss{1.55}{0.03}\\
		\bottomrule
	\end{tabular}
	\caption{Simulation experiment: average squared loss (First two columns) and $\loss_{\rm PD}$ (Last two columns) error of the proposed structured prediction (SP) approach and the KRLS baseline on learning the inverse of a PD matrix for increasing matrix dimension.\label{tab:simulation}}
\end{table}
For the experiments reported in the following we compared the performance of the manifold structured estimator minimizing the loss $\loss_{\textrm{PD}}$ and a Kernel Regularized Least Squares classifier (KRLS) baseline (see Appendix~\ref{sec:KRLS-app}), both trained using the Gaussian kernel $k(x,x') = \exp(-\|x-x'\|^2 / 2\sigma^2)$. The matrices predicted by the KRLS estimator are projected on the PD manifold by setting to a small positive constant ($1e-12$) the negative eigenvalues. For the manifold structured estimator, the optimization problem at \eqref{eq:structestimator} was performed with the Riemannian Gradient Descent (RGD) algorithm \cite{absil2009optimization}. We refer to \cite{zhang2016first} regarding the implementation of the RGD in the case of the geodesic distance on the PD cone.\\

\noindent{\bf Learning the Inverse of a Positive Definite Matrix.} We consider the problem of learning the function $f:\ppp m\to\ppp m$ such that $f(X) = X^{-1}$ for any $X\in\ppp m$. Input matrices are generated as $X_i = U\Sigma U^\top\in\ppp m$ with $U$ a random orthonormal matrix sampled from the Haar distribution \cite{diestel2014joys} and $S\in\ppp m$ a diagonal matrix with entries randomly sampled from the uniform distribution on $[0,10]$. We generated datasets of increasing dimension $m$ from $5$ to $50$, each with $1000$ points for training, $100$ for validation and $100$ for testing. The kernel bandwidth $\sigma$ was chosen and the regularization parameter $\lambda$ were selected by cross-validation respectively in the ranges $0.1$ to $1000$
and $10^{-6}$ to $1$ (logarithmically spaced).

\begin{figure}[t]
	\CenterFloatBoxes
	\begin{floatrow}
		\begin{minipage}[t]{0.3\textwidth}
			\footnotesize
			\begin{tabular}{lc}
				\toprule
				& $\Delta$ Deg. \\
				\midrule
				KRLS      & $26.9\pm5.4$ \\
				MR\cite{steinke2010nonparametric} &$22\pm6$ \\
				SP (ours)    & $\bf 18.8\pm3.9$\\
				\bottomrule
			\end{tabular}
		\end{minipage}
		\begin{minipage}[t]{0.7\textwidth}
			\ffigbox{%
				\includegraphics[width=1\textwidth]{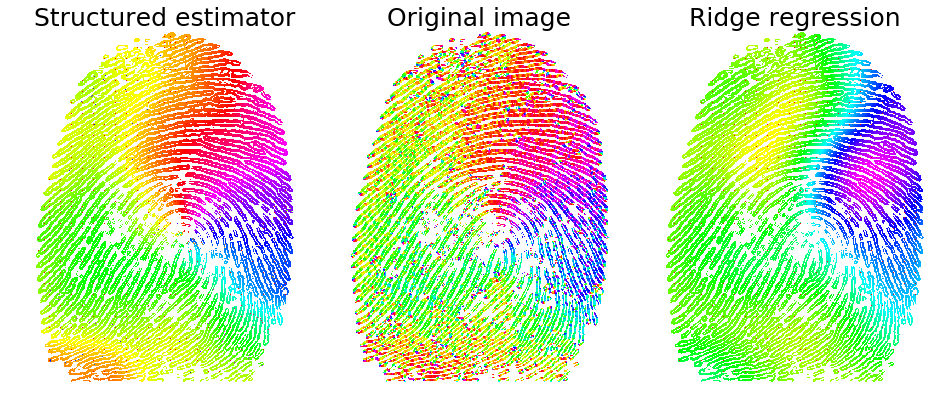}%
			}{}
		\end{minipage}
		\caption{(Left) Fingerprints reconstruction: Average absolute error (in degrees) for the manifold structured estimator (SP), the manifold regression (MR) approach in \cite{steinke2010nonparametric} and the KRLS baseline. (Right) Fingerprint reconstruction of a single image where the structured predictor achieves $15.7$ of average error while KRLS $25.3$. \label{tab:fingerprinterror}}
	\end{floatrow}
\end{figure}

\Cref{tab:simulation} reports the performance of the manifold structured estimator (SP) and the KRLS baseline with respect to both the $\loss_{\textrm{PD}}$ loss and the least squares loss (normalized with respect to the number of dimensions). Note that the KRLS estimator target is to minimize the least squares (Frobenius) loss and is not designed to capture the geometry of the PD cone. We notice that the proposed approach significantly outperforms the KRLS baseline with respect to the $\loss_{\rm PD}$ loss. This is expected: $\loss_{\rm PD}$ penalizes especially matrices with very different eigenvalues and our method cannot predict matrices with non-positive eigenvalues, as opposed to KRLS which computes a linear solution in $\R^{d^2}$ and then projects it onto the manifold. However the two methods perform comparably with respect to the squared loss. This is consistent with the fact that our estimator is aware of the natural structure of the output space and uses it profitably during learning. 
\subsection{Fingerprint Reconstruction}\label{sec:exp-fingerprints}

We consider the fingerprint reconstruction application in \cite{steinke2010nonparametric} in the context of manifold regression. Given a partial image of a fingerprint, the goal is to reconstruct the contour lines in output. Each fingerprint image is interpreted as a separate structured prediction problem where training input points correspond to the $2$D position $x\in\R^2$ of valid contour lines and the output is the local orientation of the contour line, interpreted as a point on the circumference $\S_1$. The space $\S_1$ is a manifold with squared geodesic distance $\loss_{\S_1}$ between two points $z,y\in \S_1$ corresponding to 
\eqals{\label{eq:spheregeodetic}
	\loss_{\S_1}(z,y) = \arccos\left(\scal{z}{y}\right)^2
}
where $\arccos$ is the inverse cosine function. In Table \ref{tab:manifoldinfo} we show the computation of the structured loss, the gradient of the structured loss and the chosen retraction for the sphere manifold. We compared the performance of the manifold structured estimator proposed in this paper with the manifold regression approach in \cite{steinke2010nonparametric} on the FVC fingerprint verification challenge dataset\footnote{\url{http://bias.csr.unibo.it/fvc2004}}. The dataset consists of $48$ fingerprint pictures, each with $\sim1400$ points for training, $\sim1000$ points for validation and the rest ($\sim25000$) for test. 

\Cref{tab:fingerprinterror} reports the average absolute error (in degrees) between the true contour orientation and the one estimated by our structured prediction approach (SP), the manifold regression (MR) in \cite{steinke2010nonparametric} and the KRLS baseline. Our method outperforms the MR competitor by a significant margin. As expected, the KRLS baseline is not able to capture the geometry of the output space and has a significantly larger error of the two other approaches. This is also observed on the qualitative plot in \Cref{tab:fingerprinterror} (Left) where the predictions of our SP approach and the KRLS baseline are compared with the ground truth on a single fingerprint. Output orientations are reported for each pixel with a color depending on their orientation (from $0$ to $\pi$). While the KRLS predictions are quite inconsistent, it can be noticed that our estimator is very accurate and even ``smoother'' than the ground truth. 
\begin{table}[t]
	\begin{tabular}{lcc}	
		\toprule
		& KRLS & SP (Ours)\\
		\midrule
		Emotions    &$0.63$ & $\bf 0.73$\\
		CAL500 		&$\bf 0.92$ & $\bf 0.92$ \\
		Scene       &$0.62$ & $\bf 0.73$\\
		\bottomrule
	\end{tabular}
	\caption{Area under the curve (AUC) on multilabel benchmark datasets \cite{tsoumakas2009mining} for KRLS and SP.\label{tab:multilabel}}
\end{table}

\subsection{Multilabel Classification on the Statistical Manifold}

We evaluated our algorithm on multilabel prediction problems. In this context the output is an $m$-dimensional histogram, i.e. a discrete probability distribution over $m$ points. We consider as manifold the space of probability distributions over $m$ points, that is the $m$-dimensional simplex $\Delta^m$ endowed with the {\em Fisher information metric} \cite{amari2007methods}. We will consider $\Y = \Delta^m_{\epsilon}$ where we require $y_1,\dots, y_m \geq \epsilon$, for $\epsilon > 0$. In the experiment we considered $\epsilon = 1e-5$. The geodesic induced by the Fisher metric is,
$d(y,y') =  \arccos \left(\sum_{i=1}^m \sqrt{y_i y_i'}\right)$\cite{nielsen2017clustering}.
This geodesic comes from applying the map $\pi \colon \Delta^m \to \S_{m-1}, \; \pi(y) = (\sqrt{y_1}, \ldots, \sqrt{y_{m+1}})$ to the points $\{y_i\}_{i=1}^{n} \in \Delta^m$. This results in points that belong to the intersection of the positive quadrant $\R^m_{++}$ and the sphere $\S_{m-1}$. We can therefore use the geodetic distance on the Sphere and gradient and retraction map described in \Cref{tab:manifoldinfo}.
We test our approach on some of the benchmark multilabel datasets described in \cite{tsoumakas2009mining} and we compare the results with the KRLS baseline. We cross-validate $\lambda$ and $\sigma$ taking values, respectively, from the intervals $[1e-6, 1e-1]$ and $[0.1, 10]$.  We compute the area under curve (AUC)~\cite{srinivasan1999note} metric to evaluate the quality of the predictions, results are shown in Table \ref{tab:multilabel}.
\section{Conclusions}
In this paper we studied a structured prediction approach for manifold valued learning problems. 
In particular we characterized a wide class of loss functions (including the geodesic distance) for which we proved the considered algorithm to be statistically consistent, additionally providing finite sample bounds under standard regularity assumptions. 
Our experiments show promising results on synthetic and real data using two common manifolds: the positive definite matrices cone and the sphere. With the latter we considered applications on fingerprint reconstruction and multi-labeling. The proposed method leads to some open questions. From a statistical point of view it is of interest how invariants of the manifold explicitly affect the learning rates, see \Cref{rem:cdelta}. From a more computational perspective, even if experimentally our algorithm achieves good results we did not investigate convergence guarantees in terms of optimization.
{
	\bibliographystyle{plain}
	\bibliography{../biblio/biblio}
}

\newpage

\appendix
\section*{Appendix}

The appendix of this work is organized in the following sections:
\begin{itemize}
\item A SELF property for smooth functions defined on manifolds (\Cref{thm:smooth-self}).
\item B Proof of SELF property for squared geodesic distances (\Cref{thm:geodesic-self}).
\item C Generalization bounds for the structured estimator with squared geodesic loss (\Cref{thm:generalization-bound}).
\item D Basic definitions and concepts for Riemannian manifolds.
\item E Riemannian gradient descent algorithm.
\item F A note on KRLS for the experiments in \Cref{sec:exp-positive-definite} on PD matrices.
\end{itemize}

\section{Proof of \Cref{thm:smooth-self}}\label{sect:proof}

We prove here intermediate results that will be key to prove \Cref{thm:smooth-self}. We refer to \cite{lee2003smooth} for basic definitions on manifolds and to \cite{aronszajn1950theory} for an introduction on reproducing kernel Hilbert spaces (RKHS).\\

\noindent{\bf Notation and Definitions.} We recall here basic notations and definition that will be used in the following. Given a smooth manifold $\M$, for any open subset $U\subseteq\M$ we denote by $\cinf(U)$ the set of smooth functions on $U$ and with $\ccinf(U)$ the set of {\em compactly supported} smooth functions on $U$, namely functions such that the closure of their support is a compact set. For a compact subset $N\subset\M$ we denote by $\ccinf(N)$ the set of all functions $h:N\to\R$ that admit an extension $\bar h\in\ccinf(\M)$ such that $\bar h|_N = h$ and its support is contained in $N$, namely it vanishes on the border of $N$. Finally, for any subset $N$ of $\M$ we denote $\cinf(N)$ the set of all functions that admit a smooth extension in $\cinf(\M)$. 

In the following, a central role will be played by tensor product of topological vector spaces \cite{treves2016topological}. In particular, for a Hilbert space $\hh$, we will denote $\hh\otimes\hh$ the closure of the tensor product between $\hh$ and itself with respect to the canonical norm such that $\|h\otimes h'\|_{\hh\otimes\hh} = \|h\|_\hh\|h'\|_\hh$ for any $h,h'\in\hh$. Moreover, to  given a compact set $N\subset\R^d$, we recall that $\ccinf(N)\hat\otimes_\pi\ccinf(N)$ denotes the completion of the topological tensor product between $\ccinf(N)$ and itself with respect to the projective topology (see \cite{treves2016topological} Def. $43.2$ and $43.5$). In the following, for simplicity, we will denote this space with $\ccinf(N)\otimes\ccinf(N)$ with some abuse of notation. Finally, for any subset $\Y\subseteq\M$ and space $\cal F$ of functions from $\M$ to $\R$ we denote by $\cal F|_\Y$ the space of functions from $\Y$ to $\R$ that admit an extension in $\cal F$. In particular not that $\cinf(\Y) = \cinf(\M)|_\Y$.

\subsection{Auxiliary Results} 
We are ready to prove the auxiliary results. 

\begin{lemma}\label{lem:loss-in-hh-is-self}
	Let $\M$ be a topological space, $Y\subseteq\M$ be a compact subset and $\hh$ a reproducing kernel Hilbert space of functions on $\M$ with kernel $K:\M\times\M\to\R$ such that there exists $\kappa>0$ for which $k(y,y)\leq \kappa^2$ for any $y\in\Y$. Then, for any $\bar h\in\hh\otimes\hh$, its restriction  to $\Y\times\Y$, $h = \bar h|_{\Y\times\Y}$ is SELF. 
\end{lemma}

\begin{proof}
	Denote $K_y = k(y,\cdot)\in\hh$ for every $y\in\M$. Then the space $\hh\otimes\hh$ is an RKHS with reproducing kernel $\bar K:(\M\times\M)\times(\M\times\M)\to\R$ such that $\bar K((y,z),(y',z')) = K(y,y')K(z,z')$ for any $y,y',z,z'\in\M$ (see e.g. \cite{aronszajn1950theory}). In particular $\bar K_{(y,z)} = K_y \otimes K_z$. Let now $\bar h:\M\times\M\to\R$ be a function in $\hh\otimes\hh$. In particular, there exist a $V\in\hh\otimes\hh$ such that $\scal{V}{K_y\otimes K_z}_{\hh\otimes\hh} = \bar h(y,z)$ for any $y,z\in\Y$ (reproducing property). Note that $\hh\otimes\hh$ is isometric to the space of Hilbert-Schmidt operators from $\hh$ to itself, with inner product corresponding to $\scal{A}{B}_{\hh\otimes\hh} = \scal{A}{B}_\HS = \tr(A^*B)$ for any $A,B\in\hh\otimes\hh$, with $A^*$ denoting the conjugate of $A^*\in\hh\otimes\hh$.  
	Therefore, for any $y,z\in\Y$ we have
	\eqal{\label{eq:def-V}
		\bar h|_{\Y\times\Y}(y,z) =  \bar h(y,z) = \scal{V}{K_y\otimes K_z}_{\hh\otimes\hh} = \tr(V^*K_y\otimes K_z) = \scal{K_z}{V^* K_y}_\hh.
	}
	Since $K_y$ is bounded in $\hh$, for $y \in \Y$ and the operator norm of $V$ is bounded by its Hilbert-Schmidt norm, namely $\|V\|\leq\|V\|_{\HS}$, we can conclude that $h = \bar h|_{\Y\times\Y}$ is indeed SELF. 
\end{proof}

\begin{lemma}\label{lem:h-contains-cinfty0}
	Let $\M$ satisfy \Cref{asm:without-boundary}. Then there exists a reproducing kernel Hilbert space of functions $\hh$ on $\M$, with bounded kernel, such that $\ccinf(\M) \subseteq \hh$. 
\end{lemma}
\begin{proof}
	Let $H_s^2(\M)$ denote the Sobolev space on $\M$ of squared integrable functions with smoothness $s>0$ (see \cite{hebey2000nonlinear} for the definition of Sobolev spaces on Riemannian manifolds). By construction (see page 47 of \cite{hebey2000nonlinear}), $\ccinf(\M) \subset H_s^2(\M)$ for any $s>0$. To prove this Lemma, we will show that $H_s^2(\M)$ is an RKHS for any $s>d/2$. The proof is organized in two steps.\\
	
	\noindent{\bf Step 1: $H^2_s(\M)$ is continuously embedded in $C(\M)$.} By \Cref{asm:without-boundary}, we can apply Thm.~$3.4$ in \cite{hebey2000nonlinear} (see also Thm.~$2.7$~\cite{hebey2000nonlinear} for compact manifolds), which guarantees the existence of a constant $C > 0$ (see last lines of the proofs for its explicit definition) such that
	$$\sup_{y \in \M}|f(y)| \leq C \|f\|_{\H^2_s(\M)},$$
	for any $y \in \M$ and $f \in \H^2_s(\M)$.\\
	
	\noindent{\bf Step 2: Constructing $\hh$ from $H^2_s(\M)$.} Prop.~2.1 of \cite{hebey2000nonlinear} proves that there exists an inner product, that we denote by $\scal{\cdot}{\cdot}_\hh$, whose associated norm is equivalent to $\|\cdot\|_{H^2_s(\M)}$ and such that the space $\hh = (H^2_s(\M), \scal{\cdot}{\cdot}_\hh)$ is a Hilbert space. 
	
	Now, for any $y \in \M$ denote by $e_y : \hh \to \R$, the linear functional corresponding to the evaluation, that is $e_y(f) = f(y)$. Now by Step 1, we have that the linear functional $e_y$ is uniformly bounded and so continuous, indeed,
	$$
	|e_y(f)| = |f(y)| \leq C \|f\|_{\hh}, \quad \forall f \in \hh.
	$$
	So by the Riesz representation theorem $e_y \in \hh$ and so $\hh$ is a reproducing kernel Hilbert space, with kernel $k(y,y') = \scal{e_y}{e_{y'}}_\hh$, 
	(see \cite{aronszajn1950theory}, page 343, for more details).
	Note finally that the kernel is bounded since
	$$\|e_y\|_\hh = \sup_{\|f\|_\hh \leq 1} |\scal{e_y}{f}_\hh| = \sup_{\|f\|_\hh \leq 1}  |e_y(f)| \leq C,$$
	and therefore $k(y,y') \leq \|e_y\|_\hh\|e_{y'}\|_\hh \leq C^2$.
\end{proof}

In the following, let $A \subseteq \{f: U \to S\}$ and $B \subseteq \{g: V \to S\}$, with $U, V, S$ topological spaces. We denote $A \cong B$ if there exists an invertible map $q: U \to V$, such that $B = A \circ q^{-1}$ and $A = B \circ q$.

\begin{lemma}[see also \cite{mrvcun2005isomorphisms,nestruev2006smooth}]\label{lem:ccinftyN-isomorphic-Rd}
	Let $U$ be a geodesically convex open subset of a $d$-dimensional complete Riemannian manifold $\M$ without border, then there exists a smooth map $q: U \to \R^d$ with smooth inverse, such that
	$$\cinf(U) \cong \cinf(\R^d), \qquad \textrm{and} \qquad \ccinf(U) \cong \ccinf(\R^d)$$
	moreover for any compact set 
	$\Y \subset U$ there exists a compact set $R \subset \R^d$ such that $R = q(\Y)$ and the map $s$, that is the restriction of $q$ to $\Y \to R$, guarantees
	$$\cinf(\Y) \cong \cinf(R), \qquad \textrm{and} \qquad \ccinf(\Y) \cong \ccinf(R)$$
\end{lemma}
\begin{proof}
	By \Cref{lem:d-smooth}, there exists a point $p\in U$ such that $d(p,\cdot)$ admits all directional derivatives in all points $q\in U$ (it is, in fact in $C^\infty(U)$). 
	We are therefore in the hypotheses of Thm. $2$ in \cite{wolter1979distance}, from which we conclude that there exists a smooth diffeomorphism between $U$ and $\R^d$ (with smooth inverse). Denoting by $q$ the diffeomorphism between $U$ and $\R^d$, for any function $f \in \cinf(U)$, we have $f \circ q^{-1} \in \cinf(\R^d)$, so $\cinf(U) \circ q^{-1} \subseteq \cinf(\R^d)$ and for any function $g \in \cinf(\R^d)$ we have $g \circ q \in \cinf(U)$, so $\cinf(\R^d) \circ q \subseteq \cinf(U)$. Finally we recall that if $A \subseteq B$, then $A \circ p \subseteq B \circ p$ for any set $A, B$ and any map $p$ applicable to $A, B$. Then
	$$\cinf(U) = \cinf(U) \circ q^{-1} \circ q \subseteq \cinf(\R^d) \circ q \subseteq \cinf(U)$$
	and so $\cinf(N) \cong \cinf(\R^d)$. The same reasoning holds $\ccinf(U) \cong \ccinf(\R^d)$. 
	
	Analogously, the smooth diffeomorphism $q$ maps compact subsets of $U$ to compact subsets of $\R^d$. Denote by $R \subset \R^d$ the compact subset that is $q(\Y)$, the image of $\Y \subseteq U$ a compact subset of $U$, then $s$ is the restriction of $q$ to $\Y \to R$. By the same reasoning as above, we have that $\cinf(\Y)\cong\cinf(R)$ via $s$.
\end{proof}

\begin{lemma}\label{lem:CYY-in-CNxCN}
	Let $U$ be a open geodesically convex subset of a complete Riemannian $d$-dimensional manifold $\M$ and $\Y$ a compact subset of $U$, then there exists a compact subset $N\subseteq U$ such that $\Y$ belongs to the interior of $N$ and 
	$$\cinf(\Y\times \Y) \subseteq (\ccinf(N) \otimes \ccinf(N))|_{Y \times Y}.$$
	Moreover, $\cinf(\Y)\subseteq\ccinf(N)|_\Y$.
\end{lemma}

\begin{proof}
We first consider the real case $U = \M = \R^d$ with Euclidean metric. By Cor. 2.19 in \cite{lee2003smooth}, for any open subset $V\subset\R^d$ we have that any $f\in\cinf(\Y)$ admits an extension $\tilde f\in\cinf(\R^d)$ such that $\tilde f|_\Y = f$ and $\supp \tilde f \subset \ccinf(V)$. Then, since $\Y$ is bounded (compact in a complete space), there exists a bounded open set $V$ containing $\Y$. Let $N = \overline V$ the closure of $V$. $N$ is a compact set as well and contains $\Y$ in its interior. In particular, since for any $f\in\cinf(\Y)$ the extension $\tilde f$ has support contained in $V\subset N$, this shows that $\cinf(\Y)\subseteq \ccinf(N)$. Analogously we have $\cinf(\Y\times\Y) \subseteq \ccinf(N \times N)$. 

Now, by Thm. $51.6~(a)$ in \cite{treves2016topological}, we have that 
\eqals{
	\ccinf(N)\otimes\ccinf(N) \cong \ccinf(N \times N).
}
which concludes the proof in the real setting. The proof generalizes trivially to the case where $U$ is an open geodesically convex subset of a complete Riemannian manifold thanks to the isomorphisms between spaces of smooth functions provided by \Cref{lem:ccinftyN-isomorphic-Rd}.
\end{proof}

\subsection{Proof of \Cref{thm:smooth-self}}

\noindent For the following results we need to introduce the concept of {\em cut locus}. For any $y\in\M$, denote by $\textrm{Cut}(y) \subseteq \M$ the {\em cut locus} of $y$ the closure of the set of points $z\in\M$ that are connected to $y$ by more than one minimal geodesic (see \cite{gallot1990riemannian,sakai1996riemannian}). For any $y\in\Y$ we have $y\in\M\setminus\cutlocus(y)$, see e.g. Lemma $4.4$ in \cite{sakai1996riemannian}.

Finally we refine \Cref{asm:Y-geodesically-convex} to avoid pathological cases. Indeed a geodesically convex set can still have conjugate points on the boundary. To avoid this situation we restate \Cref{asm:Y-geodesically-convex} as follows \\

{\noindent\bf Assumption~2'}~~{\em
	$\widetilde{M}$ is an open geodesically convex subset of the manifold $\M$ and  $\Y$ is a compact subset of $\widetilde{M}$.
}\\

\begin{proof}[Proof of \Cref{thm:smooth-self}]

	
	By Asm.~2', let $\wm$ be an open geodesically convex subset of $\M$ such that $\Y \subset \wm \subseteq \M$. Apply \Cref{lem:CYY-in-CNxCN} and let $N\subseteq\wm$ be a compact set such that $\Y$ is contained in the interior of $N$, namely 
	\eqals{
		\cinf(\Y)\subseteq\ccinf(N)|_\Y \subseteq \ccinf(\M)|_\Y \subseteq\hh|_\Y.
	}
	Then, by applying again \Cref{lem:CYY-in-CNxCN} we have 
	\eqal{
		\cinf(\Y\times\Y) \subseteq(\ccinf(N)\otimes\ccinf(N))|_{\Y\times\Y} \subseteq (\hh\otimes\hh)|_{\Y\times\Y}.
	}

	%
	Therefore we conclude that for any $h\in\cinf(\Y)$, there exists $\bar h:\M\times\M\to\R$ with $\bar h\in\hh\otimes\hh$ and $h = \bar h |_{\Y \times \Y}$. Finally we apply \Cref{lem:loss-in-hh-is-self} to $\bar h$, which guarantees $h$ to be SELF.
\end{proof}

\section{Proof of \Cref{thm:geodesic-self}}
We prove a preliminary result. 

\begin{lemma}\label{lem:d-smooth}
	Let $\M$ be a Riemannian manifold and $N$ be a geodesically convex subset of $\M$. Then, 
	$$\geod^2|_{N \times N} \in C^\infty(N \times N).$$
\end{lemma}
\begin{proof}
	For any $y\in\M$, denote $\textrm{Cut}(y) \subseteq \M$ the {\em cut locus} of $y$, that is the set of poin{}ts in $z\in\M$ that are connected by more than one minimal geodesic curve with $y$ (see \cite{gallot1990riemannian,sakai1996riemannian}).
	Let $\textrm{Cut}(\M) = \bigcup_{y \in \M} (\{y\} \times \textrm{Cut}(y)) ~ \subseteq \M \times \M$. Then, then the squared geodesic distance is such that (see e.g. \cite{villani2008optimal}, page 336)
	$$\geod^2 \in C^\infty(\M \times \M \setminus \textrm{Cut}(\M)).$$
	Now note that by definition of geodesically convex subset $N\subseteq\M$, for any two points in $N$ there exist one and only one minimizing geodesic curve connecting them. Therefore, $N\times N ~\cap~ \textrm{Cut}(\M) = \emptyset$ and consequently $N\times N \subseteq \M \times \M \setminus \textrm{Cut}(\M)$. We conclude that the restriction of $\geod^2$ on $N \times N$ is $C^\infty$ as required. 
\end{proof}

\begin{proof}[Proof of \Cref{thm:geodesic-self}]
	By \Cref{lem:d-smooth}, under \Cref{asm:without-boundary} and \Cref{asm:Y-geodesically-convex}, the squared geodesic distances is smooth. The desired result is then obtained by applying \Cref{thm:smooth-self}.
\end{proof}

\section{Proof of \Cref{thm:generalization-bound}}


\begin{proof}
	The theorem is proved by combining \Cref{thm:smooth-self} with Thm. $5$ in \cite{ciliberto2016}. To characterize the constant $\mathsf{c}_\loss$ we need an extra step.
	
	Under \Cref{asm:without-boundary} and \Cref{asm:Y-geodesically-convex} and the smoothness of $\loss$, we can apply \Cref{thm:smooth-self}, which characterizes $\loss$ as SELF. According to the proof of \Cref{thm:smooth-self} and in particular of \Cref{lem:loss-in-hh-is-self}, for any $y,z\in\Y$ we have 
	\eqals{
		\loss(y,z) = \scal{\psi(y)}{V\psi(z)}_\hh
	}
	where $\hh = H_{s}^2(\M)$ with $s > d/2$, $\psi(y) = K_y(\cdot)$ where $K:\M\times\M\to\R$ is the reproducing kernel associated to $\hh$ and $V:\hh\to\hh$ is the operator defined in \Cref{eq:def-V}. In particular, by the isometry between the tensor space $\hh\otimes\hh$ and the space of Hilbert-Schmidt operators from $\hh$ to $\hh$, we have
	\eqals{
		\|V\|_\HS = \|\loss\|_{\hh\otimes\hh}.
	}
	To conclude, since $\loss$ is SELF, the following generalization bound in Thm. $5$ from \cite{ciliberto2016} 
	\eqals{
		\E(\fhat) - \E(\fstar) \leq \|V\|~\mathsf{q}~ \tau^2~ n^{-\frac{1}{4}}
	}
	holds with probability at least $1-8e^{-\tau}$. Here, $\|V\|$ denotes the operator norm of $\|V\|$ and $\mathsf{q}$ is a constant depending only on $\Y$ and the distribution $\rho$ (see end of proof of Lemma $18$ for additional details). Finally, we recall, by the relation between the operator and Hilbert-Schmidt norm, that $\|V\|\leq\|V\|_\HS = \|\loss\|_{\hh\otimes\hh} = \mathsf{c}_\loss$.

\end{proof}

\section{Differential geometry definitions}\label{sec:diffgeo-app}

A Riemannian manifold $(\M, g)$ of dimension $n$ is a topological space $\M$ such that every point $y \in \M$ has a neighbourhood which is homeomorphic to an open set in Euclidean space $\R^n$ and $g$ is a collection of inner product defined in every tangent space $T_{y}\M$ of every point $y \in \M$. Intuitively, the tangent space $T_y \M$ is an approximation of a neighbourhood of $y \in M$ that has a vector space structure. We will denote the inner product of $u, v \in T_y\M$ as $ \langle u, v \rangle_y$. Thanks the inner product structure in every tangent space of the manifold we can compute gradients of functions $f \colon \M \to \R$ that we will denote with $\nabla_{\M} f \colon \mathcal{F}_y (\M) \to T_{y}\M$. Where $F_y(\M)$ is the set of smooth real-valued functions defined on a neighbourhood of $y$.

For any $y_0, y_1 \in \M$ and $v \in T_y\M$ there is a unique smooth \textit{geodesic} curve $\gamma \colon [0,1] \to \M$ such that $\gamma(0) = y_0, \; \gamma(1)=y_1$ and $\frac{d}{dt} \gamma(0) = v$, this curve locally minimizes the path between $y_0$ and $y_1$. Given the geodesic between $y_0$ and $y_1$ with derivative  $\frac{d}{dt} \gamma(0) = v$, the exponential map $Exp_{y_0} \colon T_{y_0}\M \to \M$ maps vector $v \in T_{y_0}\M$ to $y_1$.
A retraction $R_{y} \colon T_y \M \to \M$,  is a first order approximation of the exponential map. Exponential maps are retractions.
\begin{figure}
	\centering
	\includegraphics[width=0.6\columnwidth]{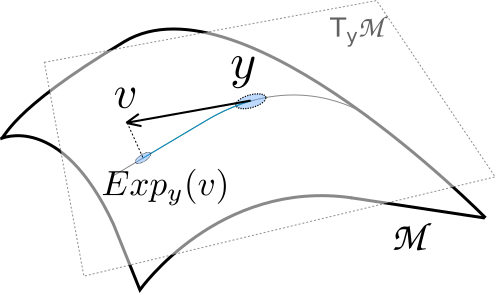}
	\caption{Pictorial representations of the exponential map.}
\end{figure}

\section{Riemannian Gradient Descent}\label{sec:algs-app}
In this section we report fully the algorithm Riemannian Gradient Descent.
\begin{algorithm}[H]
	\caption{Riemannian gradient descent}
	\begin{algorithmic}[1]
		\scriptsize
		\REQUIRE number of iterations $T$, step size $\eta$, initial point $y_0$
		\FOR{$t=0,\ldots,T-1$}
		\STATE $v_t = \nabla_{\mathcal{M}}\sum\limits_{i=1}^{n} \alpha_i(x) \L(y_t, y_i) $
		\STATE $y_{t+1} \leftarrow R_{y_t} (\eta_t \; v_t)$
		\ENDFOR
		\RETURN $y_T$
	\end{algorithmic}
\end{algorithm}

\section{Kernel Regularized Least Squares estimator for Positive definite matrices}\label{sec:KRLS-app}
We consider the case where we want to use KRLS estimators to predict a positive definite matrix given a data set $\{x_i, y_i\}_{i=1}^{n}$.
The KRLS estimator $f \colon \R^d \to \R$ is a function defined as $f(x) = \sum_{i=1}^{n} k(x, x_i) w_i$, where $k \colon \R^d \times \R^d \to \R$ is a reproducing kernel and $w = [w_1, \ldots, w_n] \in \R^n$ are constant weights  computed by solving the problem:

$$
\min_{f \in \mathcal{H}} \frac{1}{n} \sum\limits_{i=1}^{N} \| \hat{y}_i - \hat{K}w \|^2 +  \lambda \|w\|^{2}
$$
$\hat{K} \in \mathbb{R}^{n \times n}$ is the kernel matrix whose elements are defined as $ (K)_{ij} = k(x_i, x_j)$.

To predict a positive definite matrix $y \in \mathbb{P}_{d}^{++}$, a KRLS estimator is learned for every element of the flattened matrix $vec(y) \in \R^{d^2}$. Suppose $j \in \{1, \ldots, d^2\}$ is the index of the $j$-th component of $vec(y)$ that we want to predict, then we want to learn the estimator $f^{(j)}(x)= \sum_{i=1}^{n} k(x, x_i) w_i^{(j)} $. The corresponding problem has labels $\hat{y}^{(j)} = [ vec(y_1)_j, \ldots,  vec(y_{d^2})_j] $ and we solve for $w^{(j)} = [ w_1^{(j)}, \ldots, w_n^{(j)} ]$. 
Indeed we compute $d^2$ estimator to predict $vec(f) = [f^{(1)}(x), \ldots, f^{(d^2)}]$ and then recover $y$ from its vectorized form. Once the matrix is predicted we enforce it to be positive definite by performing a spectral decomposition and setting the negative eigenvalues to a small positive constant.\\
In general, when doing structured predictions with KRLS approach, it is necessary to project the outcome of the prediction on the desired manifold.
\end{document}